\documentclass[10pt, conference]{IEEEtran}
\usepackage[usenames, dvipsnames]{color}
\usepackage{subcaption}
\usepackage[cmex10]{amsmath}
\usepackage{amsthm}
\usepackage{amsfonts}
\usepackage{amssymb}
\usepackage{algpseudocode}
\usepackage{algorithmicx}
\usepackage{algorithm}
\usepackage{mathtools}
\newtheorem{theorem}{Theorem}
\newtheorem{lemma}[theorem]{Lemma}

\algnewcommand{\LineComment}[1]{\State \(\triangleright\) #1}

\begin{document}
\title{Trend Detection based Regret Minimization for Bandit Problems}

\author{\IEEEauthorblockN{Paresh Nakhe}
\IEEEauthorblockA{Max Planck Institut f{\"u}r Informatik,\\
Saarland University\\
Saarbr{\"u}cken, Germany\\
pnakhe@mpi-inf.mpg.de}
\and
\IEEEauthorblockN{Rebecca Reiffenh{\"a}user\\}
\IEEEauthorblockA{Dept. of Computer Science,\\
RWTH Aachen University\\
Aachen, Germany\\
rebeccar@cs.rwth-aachen.de}
}
\maketitle

\begin{abstract}
We study a variation of the classical multi-armed bandits problem. In this problem, the learner has to make a sequence of decisions, picking from a fixed set of choices. In each round, she receives as feedback only the loss incurred from the chosen action. Conventionally, this problem has been studied when losses of the actions are drawn from an unknown distribution or when they are adversarial. In this paper, we study this problem when the losses of the actions also satisfy certain structural properties, and especially, do show a trend structure. When this is true, we show that using \textit{trend detection}, we can achieve regret of order $\tilde{O} (N \sqrt{TK})$ with respect to a switching strategy for the version of the problem where a single action is chosen in each round and $\tilde{O} (Nm \sqrt{TK})$ when $m$ actions are chosen each round. This guarantee is a significant improvement over the conventional benchmark. Our approach can, as a framework, be applied in combination with various well-known bandit algorithms, like Exp3. For both versions of the problem, we give regret guarantees also for the \textit{anytime} setting, i.e. when length of the choice-sequence is not known in advance. Finally, we pinpoint the advantages of our method by comparing it to some well-known other strategies. 
\end{abstract}

\section{Introduction}
Consider the following problem: Suppose you own an apparel store and have purchased a fixed number of ad slots on some website, say Facebook. For every time someone visits the website, you can choose a set of ad impressions to display. Let's assume that an ad here consists of an image of a clothing item and that each image is associated with a click-through-rate unknown to you. Your goal is to choose images to display such that cumulative click-through-rate is maximized. How would you choose these images? This problem comes under the domain of reinforcement learning and more specifically, multi-armed bandit learning. Contrary to supervised learning (and most of current research in statistical pattern recognition and artificial neural networks), multi-armed bandit learning is characterized by its \textit{interactive nature} between an agent and an uncertain environment. Such a learning algorithm makes its next move based on the history of its past decisions and their outcomes.

More specifically, a multi-armed bandit problem is a sequential learning problem where the learner chooses an action from a set of actions in every round. Associated with each action is a loss unknown to the learner\footnote{The case with rewards is symmetric}. The goal of the learner is to minimize the loss incurred. Performance of the learning algorithm is measured by regret, compared to a certain benchmark strategy. Conventionally, in multi-armed bandit problems, the benchmark strategy is to always choose the single best action in hindsight, i.e. an action with minimum cumulative loss. This problem has been thoroughly studied in a variety of settings \cite{Auer2002, Auer-UCB, Bubeck2010, Thompson}. A distinguishing feature of such problems is the inherent exploration-exploitation trade-off. When the losses are generated from a fixed but unknown distribution, there exist algorithms \cite{Auer-UCB, Thompson, Robbins} that can achieve regret guarantee of $O(\log T)$. On the other hand, when losses for the actions are generated under no statistical assumption, or alternately when losses are generated by an adversary, best possible regret guarantee that can be achieved is $O(\sqrt{T})$ \cite{Bubeck2010}. Recently, interest has been developing \cite{Seldin, Hazan} in the question of achieving non-trivial regret guarantees when the loss model is semi-structured. Intuitively, more structure in the losses should enable more exploitation and hence allow for better regret guarantees. Along the lines of some of the recent work~\cite{Seldin}, we also define models exhibiting a certain degree of structure.

Often the real world problems do not exhibit adversarial behaviour and in many cases, the losses of different actions follow a trend structure, i.e. when one action is consistently better than others in a certain interval. 
For such more specialized models, the standard techniques prove insufficient since they do not take advantage of these properties. In this paper, we address this deficiency using the paradigm of trend detection. Broadly, we propose a strategy that keeps track of the current trend and restarts the regret minimization algorithm whenever a trend change is detected. This allows us to give regret guarantees with respect to a strategy that chooses the best action in each trend. This is a significantly stronger benchmark than the one conventionally considered. The regret guarantee with respect to this benchmark is also called switching regret.

More importantly, our proposed strategy is not specific to a particular regret minimization algorithm unlike the approaches in some recent works. In this paper, we use Exp3 as the underlying regret minimizing algorithm for its simplicity and almost optimal regret guarantee \cite{Auer2002}. However, one can use any other algorithm and analyze it in a similar way. Because of this modular structure of the algorithm, we can extend the arguments and proofs for the conventional multi-armed bandits problem to a more general setting where instead of a single action, the learner chooses multiple actions in each round \cite{Uchiya}. This problem has been studied in stochastic \cite{Kveton} and adversarial \cite{Bubeck2012} setting, but to the best of our knowledge, there are no prior works giving a switching regret guarantee for it.

One of the primary motivations for studying these bandit problems comes from the domain of recommender systems. Many web tasks such as ad serving and recommendations in e-commerce systems can be modeled as bandit problems. In these problems, the system only gets feedback for the actions chosen, for example whether the user selects the recommended items or not. Notice that these systems may recommend one or more items in each round. Motivation for using the paradigm of trend detection comes from the general observation that in many cases, the performance of actions follow a trend structure.
In the abovementioned case of an apparel store, for example, swimsuits might be the best choice during the hottest weeks of the year, or for certain time periods, it might be best to show an item a famous celebrity was recently seen wearing.

\textbf{Summary of Contribution:} For the standard $K$-armed bandit problem, we propose a new algorithm called Exp3.T. This algorithm guarantees switching regret of $\tilde{O} \left( \frac{ N \sqrt{ TK}}{\Delta_{sp}} \right) $ where $N$ is the number of trend changes and not known to the learner. $\Delta_{sp}$ indicates the degree of structure in loss model. This guarantee also holds for the anytime setting i.e. when the duration of the run, $T$, is not known in advance. We extend the analysis of this problem to the case when instead of a single action, the learner chooses a basis of uniform matroid in each round. The underlying regret minimization algorithm used in this case is OSMD \cite{Bubeck2012}. The resulting algorithm achieves switching regret of $\tilde{O} \left( \frac{ Nm \sqrt{ TK}}{\Delta_{sp}} \right) $. Finally, we provide empirical evidence for this algorithm's performance in the standard multi-armed bandit setting.

In general, our algorithm is particularly effective, i.e. gives better regret guarantees when little is known about the loss structure of actions except that the changes in the best action are not too frequent and actions are likely to be well-distinguishable. We argue that our loss models are more general and reasonable compared to the models conventionally studied: In most real world cases, we would expect to see a mixture of purely stochastic and purely adversarial data. We show that even such mixture of models allows us to give tight regret guarantees as long as the structural assumptions still hold.

\section{Previous Work}

The problem of giving regret guarantees with respect to a switching strategy has been considered previously in several works (albeit in more restricted settings), all of which consider the case when the learner chooses exactly one action in each round. Auer et al proposed Exp3.S \cite{Auer2002} along the same lines as Exp3 by choosing an appropriate regularization factor for the forecaster. This enables the algorithm to quickly shift focus on to better performing actions. For abruptly changing stochastic model, Discounted-UCB\cite{DUCB} and SW-UCB \cite{Garivier} have been proposed along the lines of UCB. In the former algorithm, a switching regret bound is achieved by progressively giving less importance to old losses while in SW-UCB, authors achieve the same by considering a fixed size sliding window. Both these algorithms achieve a regret bound of $O(\sqrt{MT \log T})$, where $M$ is the number of times the distribution changes.

Our work is closest to the algorithm Exp3.R proposed by Feraud et al \cite{Exp3R}. They also follow a paradigm very similar to trend detection and the high level ideas used in their paper are similar to ours. However, their algorithm is specific to Exp3 and only for the version of bandit problem where one chooses a single action in each round. Further, the algorithm assumes a certain gap in the performance of actions that depends on the knowledge of run time of the algorithm. This makes it inapplicable for a number of real-world scenarios.

The trend detection idea used in our algorithm is similar to the change detection problem studied in statistical analysis. Similar ideas have also been used for detection of concept drift in online classification \cite{cd1, cd2}. Common applications include fraud detection, weather prediction and in advertising. In this context, the statistical properties of target variable changes over time and the system tries to detect this change and learn the new parameters.

\section{Problem Setting}

We consider a multi-armed bandit problem with losses for $K$ distinct actions. Let the set of these $K$ actions be denoted by $[K]$. The losses of these $K$ actions can be represented by a sequence of loss vectors $\{ \textbf{x} \}_{t}$ where $x = \{ ( x_1, x_2 \cdots x_K )\}_{t}$. The loss sequence is divided into $N$ \textit{trends}. A trend is defined as a sequence of rounds where a set $S$ of $m$ actions is \textit{significantly} better than others for the duration of this trend. We say that the trend has changed when this set of actions changes. Within each trend the losses of actions in set $S$ are ``separated" from all others by a certain gap. Particularly, we consider a finer characterization of loss models than just stochastic or adversarial within a trend. Similar to the loss model introduced by Seldin et al \cite{Seldin}, we focus on models exhibiting a ``gap" in losses. Although this model is weaker than the adversarial model it still covers a large class of possible loss models. We express the gap in our loss models by an abstract term $\Delta_{sp}$, the separation parameter. Although the exact definition of this parameter changes depending on the actual model, in each case it conveys the same idea that a larger value of this parameter implies a larger gap between losses of actions in set $S$ and every other action.

\begin{enumerate}
\item \textbf{Dynamic Stochastic Regime (DSR)}: For the stochastic loss model, the loss of each action $a$ at round $t$ is drawn from an unknown distribution with mean $\mu_t^a$. Let $a^*$ and $a$ be any actions in sets $S$ and $[K] - S$ respectively. Then for all rounds $t$ in trend $\tau$, $ \mu_t^{a^*} <  \mu_t^{a} $ and the separation parameter is defined as:
\[
	\Delta_{sp} (\tau) = \min\limits_{t \in \tau} \{ \mu_t^a - \mu_t^{a^*} \}.
\]
The loss model is stochastic with separation parameter $\Delta_{sp}$, when $\Delta_{sp} = \min\limits_{\tau} \Delta_{sp}(\tau)  > 0$. The identity of best action $a^*$ changes $N$ times.
\vspace{2mm}
\item \textbf{Adversarial Regime with Gap (ARG)}: We use a modified version of the loss model introduced in \cite{Seldin}. Within each trend $\tau$, there exists a set $S$ of $m$  actions which is the best set for any interval of (sufficiently large) constant size, $C$. More precisely, let $\lambda_z(a) = \sum\limits_{t\in z} \ell_{a,t}$ be the cumulative loss of an action $a$ in interval $z$ consisting of $C$ rounds. Then for any action $a^* \in S$ and $a \in [K] - S$ we define the separation parameter for trend $\tau$ as:
\[
	\Delta_{sp}(\tau) = \min\limits_{z \in \tau} \left\lbrace \frac{\min\limits_{a' \neq a^*} \lambda_z(a') - \lambda_z(a^*)}{|z|} \right\rbrace
\]

It is the smallest average gap between any sub-optimal action and any action in set $S$ for any interval $z$ of size $C$. As in the above model, we say that a model satisfies ARG property with separation parameter $\Delta_{sp}$ when  $\Delta_{sp} = \min\limits_{\tau} \Delta_{sp}(\tau)  > 0$.

\end{enumerate}

Notice that the first trend, spanning from the first round till some round $n$, each action satisfies the gap conditions defined above for all the constituent rounds (\textsc{DSR}) or intervals of size $C$ (ARG), for the respective setting. We define $n$ to be the last such round, i.e. these conditions are violated at round $n+1$, indicating the start of a new trend.

We study two variants of this problem. In the first variant, the algorithm chooses exactly one action every round while in the other, the algorithm can choose any set of $m$ actions. For both the variants, the algorithm observes losses only of the actions chosen (or the single action chosen for the former variant). We assume the presence of an oblivious adversary which decides on the exact loss sequences before the start of the game. The sequence is of course not known to the algorithm. We also make the standard assumption that losses are bounded in the $[0, 1]$ interval.

For the problem setting as described, our goal is to design an algorithm $\mathcal{A}$ to minimize the cumulative loss incurred in the $T$ rounds that the game is played.  For the case when the algorithm chooses exactly one action every round, its performance is measured with respect to a strategy that chooses the best action in each trend. Specifically, let $I_t$ denote the action chosen by the algorithm in round $t$ and let $X_{I_t}^t$ denote the corresponding loss incurred by this action. Then the cumulative loss incurred by the algorithm is:
\[
	L_{\mathcal{A}} = \sum\limits_{t=1}^T X_{I_t}^t .
\] 
Let $I^*_{[n]}$ be the best action in trend $n$, then the loss incurred by the switching strategy described above is:
\[
	L^* = \sum\limits_{n=1}^N \sum\limits_{t=T_n}^{T_{n+1} -1} X_{I^*_{[n]}}^t ,
\]
where trend $n$ occurs in the interval $[T_n, T_{n+1} -1]$. We define regret incurred by algorithm $\mathcal{A}$ as follows:
\[
	R_T^* = L_{\mathcal{A}} - L^*.
\]

Exactly analogous definitions apply to the case when the algorithm chooses multiple actions in each round.

\underline{Assumption:} For the algorithm considered in this paper, we assume that the loss model, either stochastic or adversarial regime with gap, has separation parameter lower bounded by $4 \Delta$, a constant known to us i.e. $\Delta_{sp} \geq 4 \Delta$.

\section{The Algorithm}

The algorithm Exp3.T is composed of two primary ideas: The Exp3 algorithm and a trend detection routine. Exp3 gives almost optimal regret bound with respect to the single best action in hindsight when the loss model is adversarial. However, when the losses exhibit certain structure or when regret with respect to a stronger benchmark is desired, Exp3 proves to be insufficient. In this algorithm, we overcome this problem by identifying \textit{trends} in losses and resetting the Exp3 algorithm whenever a change in trend is detected. One advantage of using Exp3 when losses exhibit trend structure is that Exp3 is robust to changes in the losses of actions as long as the best action remains same. We exploit this property in our algorithm so that it is applicable to a large class of loss models. In the analysis we use the following regret bound given by \cite{BubeckBook}

\begin{lemma}
\label{base2}
For any non-increasing sequence $ \{ \eta \}_{t \in \mathbb{N}}$, the regret of Exp3 algorithm with $K$ actions satisfies
\[
R_T \leq \frac{K}{2} \sum\limits_{t=1}^T \eta_t + \frac{\ln K}{\eta_T} .
\]
\end{lemma}
Algorithm~\ref{Exp3.T} shows the skeleton of the procedure to achieve the desired switching regret bound. At a high level, the algorithm divides the total run into runs on smaller intervals. Within each interval the algorithm runs Exp3 (parameter $\eta$) with loss monitoring(LM) plays randomly interspersed among all rounds. The length of this interval is controlled by parameter $\gamma$. These loss monitoring plays choose different actions for a fixed number of rounds without regards to regret. The loss values collected from this process are used to give an estimation of the mean loss of each action in a given interval. The number of such plays required to give a good estimation of loss depends on the actual model under consideration and is captured by the parameter $t^*$. Based on this estimation, the trend detection module outputs with probability at least $1 - \delta$ whether the best action has changed or not, alternatively whether the trend has changed or not. 

The $Make\_Schedule(\cdot)$ procedure assigns Exp3 plays and fixed action plays to monitor loss (exactly $t^*$ many per action) randomly to rounds at the start of an interval and returns the randomly generated schedule. The random generation of schedule protects the algorithm from making biased estimates of actual losses.

\begin{algorithm}
  \caption{Exp3.T}
  \label{Exp3.T}
  
  \begin{algorithmic}[1]
   \LineComment { Parameters: $\delta$, $\gamma$ and $\eta$}
   
	\State Set interval length $ |I| = \frac{K t^*}{\gamma}$
	\For{each interval $I$}
		\State Schedule $\leftarrow$ Make\_Schedule($I$)  
		\For{$t = 1, 2 \cdots |I|$}
			\If{Schedule($t$) = Exp3 Play}
				\State Call Exp3\_play()
			\Else
				\State Call LM\_play(Schedule($t$))
			\EndIf
		\EndFor
		
		\If{trendDetection() == True}
			\State Restart Exp3
		\EndIf
	\EndFor
  \end{algorithmic}
\end{algorithm}

\subsubsection*{ Trend Detection}

In any interval, the loss monitoring component of Algorithm~\ref{Exp3.T} chooses each action a sufficient number of times and these choices are randomly distributed over the interval. The samples obtained from these plays are used to give a bound on the deviation of the empirical mean of losses from the true mean. Particularly, we use the following lemma by Hoeffding~\cite{Hoeffding} for sampling without replacement from a finite population.

\begin{lemma}
\label{base}
Let $\mathcal{X} = (x_1, x_2, \cdots x_N)$ be a finite population of $N$ real points, $X_1, X_2 \cdots X_n$ denote random sample without replacement from $\mathcal{X}$. Then, for all $\epsilon > 0$,
\begin{equation*}
\mathbb{P} \left( \frac{1}{n} \sum\limits_{i=1}^n X_i - \mu  \geq \epsilon \right) \leq exp (-2n \epsilon^2) 
\end{equation*}
where $\mu = \frac{1}{N} \sum\limits_{i=1}^N x_i$ is the mean of $\mathcal{X}$.
\end{lemma}

For each interval we maintain information about the empirical mean of losses for each action, i.e. mean over loss values actually seen by the algorithm. By Lemma~\ref{base}, all of these estimates are close to the actual mean with probability at least $1 - \delta$ where $\delta$ is a parameter of the algorithm. In case of change in trend within an interval $I$, naturally these guarantees are void as the losses do not maintain a uniform pattern. Therefore, a change in trend can be detected by comparing the empirical estimates obtained at the end of the next interval to those obtained prior to the trend change. This idea is represented in Algorithm~\ref{trend}.

\begin{algorithm}
  \caption{trendDetection()}
  \label{trend}
  
  \begin{algorithmic}[1]
   
   	\State Let $p$ be the index of the current interval
   	\State $I_p^* \leftarrow$ action with minimum empirical mean loss, $\hat{\mu}$, in interval $p$.
   	\If{$p = 1$ or $p = 2$}
   		\State return False
   	\EndIf
   	\If{$I_p^* \neq I_{p-2}^*$}
   		\State return True
   	\EndIf
   		\State return False
  \end{algorithmic}
\end{algorithm}

\section{Regret Analysis}
\label{analysis}

For ease of notation in the analysis, we define the \textit{detector complexity}, $t^*$, as the number of loss monitoring samples required for each action so that the trend detection procedure works with probability at least $1 - \delta$, provided there is no trend change in the actual interval. In what follows, we give detector complexity bounds for different models and in regret computation use $t^*$ as an abstract parameter.

\begin{lemma}
\label{lem1}
The detector complexity in dynamic stochastic regime satisfies
\[
	t^*_{DSR} = \frac{1}{2 \Delta^2} \ln \left( \frac{4K}{\delta} \right).
\]
\end{lemma}

\begin{proof}
Fix an action $a$ and an interval $I$. Let the expected reward of action $a$ on interval $I$ be given by the sequence $ \{ \mu_t^a \}_{t \in I} $ and the actual realization of rewards be given by $ \{ X_t^a \}_{t \in I} $. First we observe that the expected reward of $a$ over the interval $I$ is given by
\begin{equation*}
\mu_{a, I} = \frac{\sum_{t \in I} \mu_t^a}{|I|}.
\end{equation*}

Let the set of loss monitoring samples collected by our algorithm for action $a$ be denoted by $\mathcal{Z}_a$. The algorithm uses these samples to calculate the empirical mean of rewards for the action $a$. We denote it by $ \hat{\mu}_{\mathcal{Z}_a}$. 

\textit{Step 1:} First we show that the empirical mean of losses over the entire interval is close to the expected mean, $\mu_{a, I}$.
Let $ \{ X_t^a \}_{t \in I} $ be the sequence of actual reward realizations for arm $a$ in interval $I$. Denote by $\bar{\mu}_{a, I}$ the mean of these actual realizations. Applying Hoeffding's inequality,
\begin{equation*}
\begin{aligned}
P ( | \mu_{a, I} - \bar{\mu}_{a, I} | >  \Delta ) & \leq 2 \exp(- 2|I| \cdot  {\Delta}^2)\\
		& \leq 2 \exp(- 2 t^*_{DSR} \cdot  {\Delta}^2) = \frac{\delta}{2K}
\end{aligned}
\end{equation*}
i.e. the empirical mean of losses for action $a$ over the interval $I$ is close to the actual mean with probability at least $ 1 - \frac{\delta}{2K}$.

\textit{Step 2:} Now we show that the empirical mean of loss-monitoring samples collected for action $a$ is close to the mean of the actual realizations, $\bar{\mu}_{a, I}$. This follows from Lemma~\ref{base}:
\begin{equation*}
P( |\bar{\mu}_{a, I} - \hat{\mu}_{\mathcal{Z}_a} | > \Delta ) \leq 2 \exp( -2 t^*_{DSR} \Delta^2 ) = \frac{\delta}{2K}
\end{equation*}

Therefore, with probability at least $1 - \frac{\delta}{K}$ the mean of loss monitoring samples for any action is within $2 \Delta$ of the actual mean. By applying a union bound over all actions, with probability at least $1 - \delta$ the same guarantee holds over all actions, which in turn implies that the trend detection module can detect whether the best action has changed with the same probability.
\end{proof}

\begin{lemma}
\label{lem2}
The detector complexity in the adversarial regime with gap satisfies
\[
	t^*_{ARG} \geq \frac{(b - a)^2}{8 \Delta^2} \ln \left( \frac{2K}{\delta} \right)
\]
when the losses in the given trend are drawn from interval $[a, b]$.

\end{lemma}

\begin{proof}
The proof for this Lemma goes along the same lines as for Lemma~\ref{lem1} except that in this case we do not need step 1. Further, in this case, we can allow the empirical mean of collected samples to be within $2 \Delta$ of the actual mean of all losses in the interval instead of just $\Delta$. For this particular loss model, if additional information about the range of losses within a trend is available, then using the generalized version of Hoeffding's inequality we achieve a tighter detector complexity bound. We note if not defined otherwise, our losses are always drawn from range $[0,1]$.

\end{proof}

In the rest of the analysis, instead of  $t^*_{DSR}$ or $t^*_{ARG}$ we use the model-oblivious-parameter $t^*$.

\begin{theorem}
\label{main1}
The expected regret of Exp3.T is 
\begin{equation*}
R_T = O \left( \frac{ N \sqrt{( TK \ln K) \ln \left( TK \ln K \right)} }{\Delta_{sp}} \right).
\end{equation*}
\end{theorem}

\begin{proof}
We divide the regret incurred by Exp3.T in three distinct components; first is the regret incurred just by running and restarting of Exp3. To bound this component of total regret we use the regret bound as in Lemma~\ref{base2}. Let $F(T)$ denote the number of \textit{false trend detections} i.e. number of times when there was no change in detection but the detection algorithm still indicated a change. Then the regret incurred due to Exp3 is
\[
R_{Exp3}  \leq \frac{K}{2} \sum\limits_{t=1}^T \eta_t + \frac{(N-1 + F(T)) \ln K}{\eta_T}.
\] 

As trend detection fails with probability at most $\delta$, the expected number of false detections is at most
\[
F(T) \leq \delta \left( \frac{T}{|I|} + 1 \right).
\]

The second component of the total regret incurred is on account of intervals wasted due to delay in detection of trend change. Specifically, if the trend changes in a given interval $I$, the regret guarantee obtained as part of Exp3 is not with respect to the best action before and after trend change. As we cannot give the required guarantee for this interval, we count this interval as \textit{wasted} and account it towards regret. Secondly, since the trend detection algorithm detects the change with probability at least $ 1 - \delta$, the expected number of trend detection calls required (or alternatively the expected number of intervals) is at most $ \frac{1}{1 - \delta}$. Therefore, the total number of wasted rounds is at most
\[
R_{wasted}  \leq N \left( 1 + \frac{1}{1 - \delta} \right) |I|
\]

The third and final component of regret incurred is due to the \textit{loss monitoring plays} in each interval. No guarantee can be given about the regret incurred in these rounds and hence all such rounds are also accounted in regret. Since in each interval there are exactly $K t^*$ number of such plays, the total number of such rounds is at most 
\[
R_{loss\_monitor}  \leq  K t^* \left( \frac{T}{|I|} + 1 \right) =  \gamma T + Kt^*
\]

Putting all together, the total regret is

\begin{equation*}
\begin{multlined}
R_T \leq K \sum\limits_{t=1}^T \eta_t + \frac{(N-1 + \frac{\gamma \delta T}{K t^*}) \ln K}{\eta_T} + \\
\shoveleft[1cm] N \left( 1 + \frac{1}{1 - \delta} \right) \frac{K t^*}{\gamma} + \gamma T + Kt^*
\end{multlined}
\end{equation*}

Setting $\eta = \sqrt{\frac{\ln K}{TK}}$, $\gamma = \sqrt{\frac{K t^* \ln K}{T}}$ and $\delta = \sqrt{\frac{K}{T \ln K}}$, regret incurred by Exp3.T is

\begin{equation*}
\begin{multlined}
\vspace{2mm}
R_T \leq  \sqrt{TK \ln K} + N \sqrt{TK \ln K} + \\ 
\vspace{2mm}
\shoveleft[2cm] \sqrt{\frac{TK \ln K}{t^*}} + 2N \sqrt{\frac{TK t^*}{\ln K}} + \\
 \shoveleft[2cm] 2N \frac{K \sqrt{t^*}}{\ln K} + \sqrt{t^* TK \ln K} + Kt^*
\end{multlined}
\end{equation*}
 where $t^* = O \left( \frac{\ln \left( TK \ln K \right)}{\Delta_{sp}^2} \right)$.

Alternatively, $R_T = O \left( \frac{ N \sqrt{( TK \ln K) \ln \left( TK \ln K \right)} }{\Delta_{sp}} \right)$.
\end{proof}

\subsubsection*{Extension to Anytime Version}

The parameters derived to achieve the desired regret bound in Theorem~\ref{main1} depend on the knowledge of T, the length of the total run of the algorithm. This dependency can be circumvented by using a standard doubling trick. Particularly, we can divide the total time into periods of increasing size and run the original algorithm on each period. Since the guarantee of this algorithm rests crucially on the probability of correct trend detection, in our case we need to modify the $\delta$ parameter as well. 
\begin{algorithm}
  \caption{Anytime Exp3.T}
  \label{AnytimeExp3.T}
  \begin{algorithmic}[1]
  	\LineComment { Choose an initial estimate $T'$ of length of run }
  	
  	\For{ $i = 0, 1, 2 \cdots $}
		\State Let $T_i = 2^i T'$
		\State Set $\gamma_i = \sqrt{\frac{K t^*_i \ln K}{T_i}}$, $\delta_i = \frac{1}{T_i^{3/2}}\sqrt{\frac{K}{ \ln K}}$
		\State Run Exp3.T with parameters $\gamma_i, \delta_i$ in period $T_i$
	\EndFor
  	
  \end{algorithmic}
\end{algorithm}

\begin{theorem}
\label{minor1}
The expected regret of Anytime Exp3.T with $\eta_{i} = \sqrt{\frac{\ln K}{T_i K}}$, $\gamma_i = \sqrt{\frac{K t^*_i \ln K}{T_i}}$ and $\delta_i = \frac{1}{T_i^{3/2}}\sqrt{\frac{K}{ \ln K}}$ is $O \left( \frac{ N \sqrt{( TK \ln K) \ln \left( TK \ln K \right)} }{\Delta_{sp}} \right)$.
\end{theorem}

\begin{proof}

We follow the same steps as in the proof of Theorem~\ref{main1}. We divide the regret incurred into three different components: regret due to Exp3 algorithm, due to the wasted intervals during detection and due to the loss monitoring plays. Compared to the proof in Theorem~\ref{main1} the only difference is that here we have to sum regret of Exp3.T over multiple runs. If $T$ is the actal length of play, then the number of times we run Exp3.T is at most $\log T$. Regret due to Exp3 algorithm (running and restarting) is:
\[
	R_{Exp3} \leq \sum\limits_{i = 0}^{\lceil \log T \rceil}  \left(  \frac{K}{2} T_i \eta_{i} + \frac{(N_i - 1 + F(T_i)) \ln K}{\eta_{i}} \right)
\]
where $N_i$ and $F(T_i)$ are the number of changes in trend and number of false detections in $i$th run of Exp3.T respectively. As before, 
\begin{equation*}
\begin{aligned}
F(T_i) \leq & \enspace \delta_i \left( \frac{T_i}{|I|_i} + 1 \right) \\
	= &  \enspace \frac{1}{T_i^{3/2}}\sqrt{\frac{K}{ \ln K}} \cdot  \left(  \frac{T_i}{K t^*_i} \sqrt{\frac{K t^*_i \ln K}{T_i}}  + 1 \right) \leq \frac{2}{T_i}
\end{aligned}
\end{equation*}

Using this bound in above inequality
\begin{equation*}
\begin{aligned}
	R_{Exp3} \leq & \enspace  \sum\limits_{i=0}^{\lceil \log T \rceil}  \left[ \frac{K T_i \eta_{i}}{2}  + \frac{N \ln K}{\eta_{i}} + \frac{2 \ln K}{T_i \eta_{i}}  \right] \\
	\leq & \enspace \sqrt{K \ln K} \cdot \sum\limits_{i}^{\lceil \log T \rceil}  \left(  \frac{\sqrt{T_i}}{2} + N \sqrt{T_i} + \frac{2}{\sqrt{T_i}} \right)\\
	\leq &  \enspace C_1 \left( \sqrt{TK \ln K} + N \sqrt{ TK \ln K} \right)
\end{aligned}
\end{equation*}

The inequalities follow by using parameters $\eta_i$ and $\delta_i$ as defined in the algorithm. For ease of representation, we capture all constants with a single constant $C_1$. Regret incurred due to wasted intervals is

\begin{equation*}
\begin{aligned}
R_{wasted} \leq & \enspace \sum\limits_{i=0}^{\lceil \log T \rceil} N_i \left( 1 + \frac{1}{1 - \delta_i} \right) |I_i|\\
	\leq & \enspace \sum\limits_{i=0}^{\lceil \log T \rceil} 2N \left( 1 + \delta_i \right) \frac{K t^*_i}{\gamma_i} \\
	\leq & \enspace  \sum\limits_{i=0}^{\lceil \log T \rceil}  \frac{4 N K t^*_i}{\gamma_i}\\
	\leq & \enspace  \sum\limits_{i=0}^{\lceil \log T \rceil} N \sqrt{\frac{t_i^* T_i K}{\ln K}} \\
	\leq & \enspace C_2 \cdot \left(  N \sqrt{\frac{TK t^*}{\ln K}} \right)
\end{aligned}
\end{equation*}

Here we use the fact that $t^*_i = O(t^*)$, the detector complexity had we known $T$ apriori. All the constants involved in the above inequality are captured by $C_2$. Similarly, regret due to loss monitoring plays is:
\begin{equation*}
\begin{aligned}
R_{loss\_monitor} \leq & \enspace K \sum\limits_{i=0}^{\lceil \log T \rceil} t^*_i  \frac{T_i}{|I_i|}\\
	\leq & \enspace \sum\limits_{i=0}^{\lceil \log T \rceil} \gamma_i T_i\\
	\leq & \enspace C_3 \cdot \left( \sqrt{K T t^* \ln K} \right)
\end{aligned}
\end{equation*}
\noindent
where the constant $C_3$ captures the constants involved. Combining the above mentioned bounds we get the desired claim. This bound is only a constant factor worse than the bound proved in Theorem~\ref{main1}.

It is easy to verify that the above analysis holds if $\delta_i$ is of the order of $\delta$ and this condition is met when $T'$ is of order at least $T^{\frac 13}$. If, however, $T'$ is not a good estimate of $T$ in the above sense, the output of trend detection procedure in initial runs will not be correct with sufficiently high probability and hence aforementioned guarantees do not hold. We account for the regret incurred in the first few runs (till $T_i \geq T^{\frac 13}$) by simply disregarding all of them and consider them as \textit{wasted} rounds.
\end{proof}

The principle of trend detection and restarting of a base algorithm (Exp3 in our context) according to changes in the trend can be extended to any multi-armed bandit algorithm for adversarial setting. The final regret guarantee obtained naturally depends on the performance of the base algorithm. We notice however that due to the necessary number of exploration rounds, no base algorithm can allow us to achieve regret $o(\sqrt{T})$. In particular, by choosing an appropriate base algorithm, our framework can be adjusted to a number of different loss structures and problem settings. In the following section, we use exactly this principle to design an algorithm to minimize regret with respect to the $m$ best actions.

\section{ Extension to Top-$m$ Actions}

In this section, we show how to extend the ideas introduced above to a setting where in each round we choose $m > 1$ actions out of the $K$ available. For this variant of the problem, the Exp3 algorithm cannot be used and hence we use a more general approach proposed by Audibert et al \cite{Bubeck2012}. This approach, named Online Stochastic Mirror Descent (OSMD) is based on a powerful generalization of gradient descent for sequential decision problems. Similar to Exp3, the regret guarantee given by this technique is with respect to the best combination of actions in hindsight and holds even for adversarial losses. We refer the reader to \cite{BubeckBook} for a thorough treatment of the technique. In our proposed algorithm, OSMD.T, we use the technique as a black box and only need the final guarantee.

\begin{lemma}
\label{base3}
The regret of OSMD algorithm in the $m$-set setting with $F(x) = \sum\limits_{i=1}^{K} x_i \log x_i - \sum\limits_{i=1}^{K} x_i$ and learning rate $\eta$ satisfies 
\[
R_T \leq \frac{\eta T K}{2} + \frac{m \log \frac{K}{m}}{\eta}
\]
\end{lemma}

Here $F(x)$ is a Legendre function and is a parameter used within the OSMD technique. The trend detection algorithm in this case uses the same idea as in Algorithm~\ref{trend} except that instead of a single action we now check if the set of $m$ best actions have changed with probability at least $1 - \delta$. Even in this case, we denote by $t^*$ the number of samples needed for each action to ensure that trend detection works with above mentioned probability. Bounds derived in Lemma~\ref{lem1} and Lemma~\ref{lem2} apply in this case too.

There are only a few differences in Algorithm~\ref{osmd} as compared to Algorithm~\ref{Exp3.T}. Firstly, instead of using Exp3 for regret minimization we use the more sophisticated technique of OSMD. This algorithm gives tight regret guarantees and is polynomial time computable\footnote{The OSMD technique can also be used when there are more generic combinatorial constraints on the set of actions chosen each round. For these generic cases, the algorithm need not be poly time computable. However, for the uniform matroid case (under consideration here) it is in fact poly time computable }. Secondly, the trend detection algorithm changes slightly as mentioned above. Finally, since we choose $m$ actions in every round, we need a factor of $m$ lesser number of loss monitoring plays. Alternately, the size of an interval $I$ is chosen to be $\frac{Kt^*}{m \gamma}$.

\begin{algorithm}
  \caption{OSMD.T}
  \label{osmd}
  
  \begin{algorithmic}
   \LineComment { Parameters: $\delta$, $\gamma$ and $\eta$}
   
	\State Set interval length $ |I| = \frac{K t^*}{m \gamma}$
	\For{ each interval $I$}
		\State Schedule $\leftarrow$ Make\_Schedule($I$)  
		\For{ $t = 1, 2 \cdots |I|$ }
			\If{Schedule($t$) = OSMD Play}
				\State Call OSMD\_play()
			\Else
				\State Call LM\_play(Schedule($t$) )
			\EndIf
		\EndFor
		
		\If{trendDetection() == True}
			\State Restart OSMD
		\EndIf
	\EndFor
  \end{algorithmic}
\end{algorithm}

\begin{theorem}
\label{main2}
The expected regret of OSMD.T is
\begin{equation*}
R_T = O \left( \frac{ Nm \sqrt{ TK \ln \left( \frac{TK}{m} \right)  } }{\Delta_{sp}} \right).
\end{equation*}
\end{theorem}

\begin{proof}
The main steps of analysis in this setting are exactly the same as Theorem~\ref{main1}. The component of regret due to OSMD algorithm is
\[
	R_{osmd} \leq \frac{\eta T K}{2} + ( N - 1 + F(T) ) \frac{m \log \frac{K}{m}}{\eta},
\]

where $F(T)$ is the number of false detections as before and given by $ F(T) \leq \delta \left( \frac{T}{|I|} + 1 \right) $. This inequality follows by Lemma~\ref{base3} and considering the fact that the algorithm is restarted at most $N-1 + F(T)$ times. Following the same arguments as in Theorem~\ref{main1}, the regret incurred on account of wasted intervals is at most:
\[
R_{wasted} \leq N m \left( 1 + \frac{1}{1 - \delta} \right) |I|.
\]
Unlike Theorem~\ref{main1}, each wasted round incurs regret of $m$ instead of $1$ since we can't guarantee regret for any of the chosen actions. Finally, since both the number of loss monitoring plays and the length of an interval is reduced by a factor of $m$, the regret incurred on account of loss monitoring plays is:
\[
R_{loss\_monitoring} \leq \left \lceil \frac{K t^*}{m} \right \rceil  \cdot \left \lceil \frac{T}{|I|} \right \rceil = O \left( \gamma T \right).
\]

Putting the above bounds together,
\begin{equation}
\begin{multlined}
\vspace{2mm}
R_T =  \enspace R_{osmd} + R_{wasted} + R_{loss\_monitoring} \\
\vspace{2mm}
	\leq \enspace \frac{\eta T K}{2} + ( N - 1 + \frac{\delta \gamma mT}{K t^*} ) \frac{m \log \frac{K}{m}}{\eta} +\\
	\vspace{2mm}
	 N m \left( 1 + \frac{1}{1 - \delta} \right) \frac{K t^*}{\gamma m} + \gamma T
\end{multlined}
\end{equation}

By setting $\eta = m \sqrt{\frac{\ln \left( K/m \right)}{TK}}$, $\delta = \sqrt{\frac{mK}{T}}$ and $\gamma = \frac{1}{m} \sqrt{\frac{K t^*}{T}}$ we get

\begin{equation}
\begin{multlined}
\vspace{2mm}
R_T \leq m \sqrt{TK \ln \frac{K}{m}} + N m \sqrt{TK \ln \frac{K}{m}}  \\
\vspace{2mm}
+ \sqrt{mTK \ln \frac{K}{m} t^*} + 2Nm \sqrt{TK t^*}  \\
\vspace{2mm}
+ 2NK \sqrt{m t^*} + \frac{1}{m}\sqrt{t^* TK}
\end{multlined}
\end{equation}

Alternately, $R_T = O \left( \frac{ Nm \sqrt{ TK \ln \left( \frac{TK}{m} \right)  } }{\Delta_{sp}} \right)$.

\end{proof}

\section{Simulations}

Since our proposed algorithm comes under the domain of active learning, it is not possible to reliably use any fixed data set. Instead, to assess the performance of our algorithm we shall use artificially constructed loss generation models; a standard approach for problems of this nature.

\begin{figure*}[!ht]
    \begin{subfigure}[b]{0.48\textwidth}
        \includegraphics[width=\textwidth]{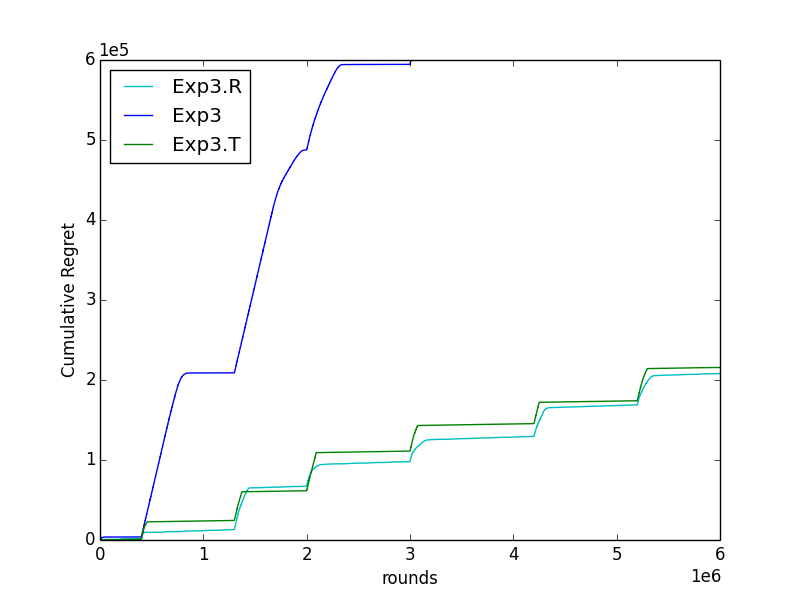}
        \caption{$K = 2$, $\Delta_{sp} = 0.55$}
        \label{fig:DSR1}
    \end{subfigure}
    ~ 
    \begin{subfigure}[b]{0.48\textwidth}
        \includegraphics[width=\textwidth]{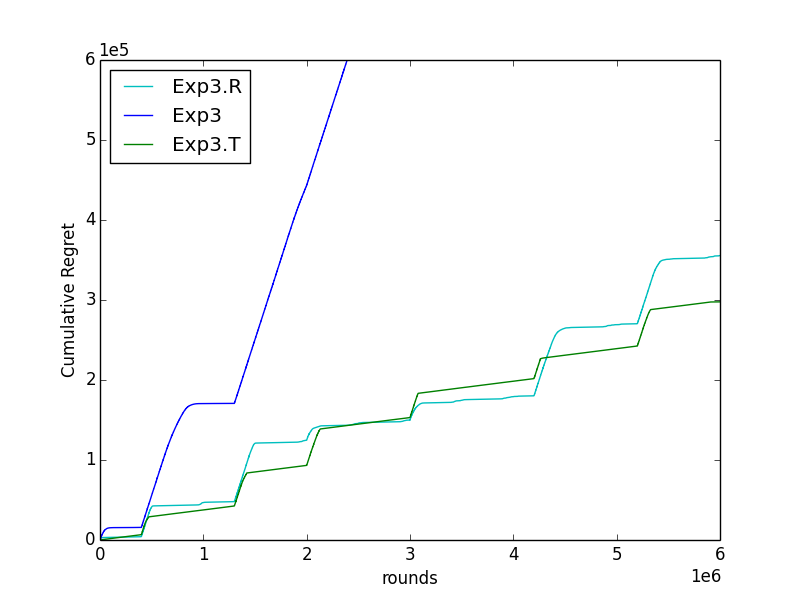}
        \caption{$K = 10$, $\Delta_{sp} = 0.40$}
        \label{fig:DSR2}
    \end{subfigure}
     \caption{DSR Model}\label{DSR}
\end{figure*}

    
\begin{figure*}[!ht]
    \begin{subfigure}[!ht]{0.48\textwidth}
        \includegraphics[width=\textwidth]{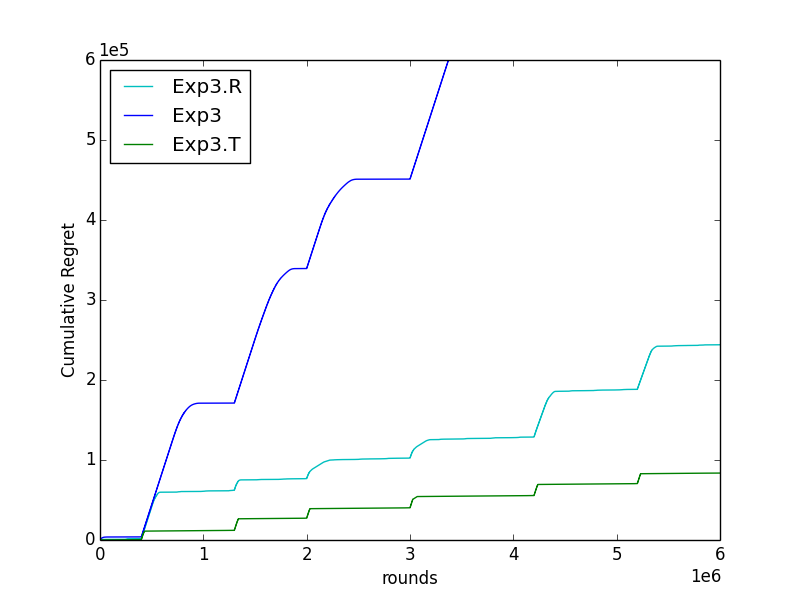}
        \caption{$K = 2$, $\Delta_{sp} = 0.40$}
        \label{fig:ARG1}
    \end{subfigure}
    ~ 
    \begin{subfigure}[!ht]{0.48\textwidth}
        \includegraphics[width=\textwidth]{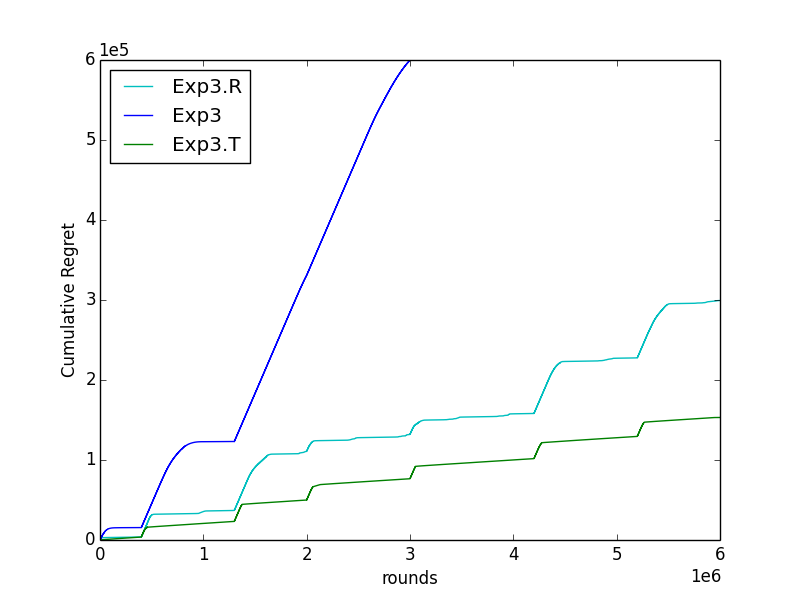}
        \caption{$K = 10$, $\Delta_{sp} = 0.30$}
        \label{fig:ARG2}
    \end{subfigure}
    \caption{ARG Model}\label{AGR}
\end{figure*}

For each of the two models introduced, we compare the performance of Exp3.T algorithm with Exp3.R\cite{Exp3R}, an algorithm closest in spirit to our work. To emphasize that we obtain \textit{switching regret} guarantee, a stronger benchmark than conventionally used, we also compare our algorithm with Exp3 i.e. the performance, measured in terms of the cumulative loss, is with respect to a switching strategy that chooses the \textit{best} action in each trend. Each experiment is run independently 10 times and the mean of the results is shown in figures.

\textbf{Experiment 1: DSR model} Within each trend, we set the bias of the best action to $0.10$ and biases of other actions for the case when $\Delta_{sp} = 0.4$ is set to $0.5$ while for the case when $\Delta_{sp} = 0.55$, they are set to $0.65$. For each of the loss models, we run the experiment with $K = 2$ and $K = 10$ actions respectively. We have constructed the dynamic stochastic loss model in our experiments as a representative of a worst case scenario i.e. we do not assume any information about the loss structure except for the separation parameter $\Delta_{sp}$ (refer Fig.~\ref{DSR}). The performance of Exp3.T is almost identical to Exp3.R, an algorithm specifically designed for stochastic model. For a smaller gap, however, our algorithm still manages to do marginally better than Exp3.R. We note here that the parameters of Exp3.R algorithm are set such that the assumptions required for the algorithm hold.
	
\textbf{Experiment 2: ARG model} We design the semi-structured property of ARG model as follows: For $\Delta_{sp} = 0.3$ case, within each trend the loss of best action is a sequence of 100 consecutive 0s followed by 100 consecutive 1s. In the same rounds, losses of sub-optimal actions are 1 and 0.6 respectively. For $\Delta_{sp} = 0.4$ case, losses of the best action are same as before but losses of sub-optimal actions are kept constant at 0.9. These loss structures are chosen as representatives of the possible instances of the ARG model. The advantage of our algorithm is clearly highlighted in this more general model. The worse performance of Exp3.R is expected since it assumes more structure than provided by the model; Exp3.T in contrast is able to exploit the little structure available and detect changes much faster.

There exists a subtle case when the guarantees presented in this paper do not hold. This happens when the length of the interval is comparable to the total run time of algorithm i.e. $O(T)$. For example, if the length of interval is $ T / 2$, then Exp3.T does not provide any switching regret guarantee since for the first two intervals Exp3.T behaves exactly like Exp3. Therefore in worst case, the regret bounds presented here are void but the bounds of Exp3 still apply.

\section{Conclusion}

We have proposed a new paradigm for regret minimization and defined a broader class of loss models where our algorithm is applicable. We have used this paradigm for the regret minimization problem when one chooses either a single action or a basis of a uniform matroid in each round. For these problems we proposed algorithms and gave switching regret bounds of $\tilde{O}( N\sqrt{TK})$ and $\tilde{O}( Nm \sqrt{TK})$ respectively. Such a paradigm is particularly suitable for regret minimization algorithms where one cannot distinguish exploration and exploitation steps, for example OSMD. Extension of this paradigm to more general problems like online linear optimization is currently in progress.

\bibliographystyle{plain}
\bibliography{references}

\end{document}